\documentclass[conference]{IEEEtran}
\usepackage{hyperref}
\usepackage{cite}
\usepackage{amsmath,amssymb,amsfonts}
\usepackage{amsthm}
\usepackage{mathtools}
\usepackage{thmtools}
\usepackage{algorithm} 
\usepackage[noend]{algpseudocode}
\newtheorem{theorem}{Theorem}
\newtheorem{lemma}[theorem]{Lemma}

\theoremstyle{definition}
\newtheorem{exmp}{Example}
\usepackage{graphicx}
\usepackage{textcomp}
\usepackage{xcolor}
\usepackage{macros2}
\def\BibTeX{{\rm B\kern-.05em{\sc i\kern-.025em b}\kern-.08em
    T\kern-.1667em\lower.7ex\hbox{E}\kern-.125emX}}
\hypersetup{
    colorlinks,
    linkcolor={black},
    citecolor={black},
    urlcolor={black}
}

\begin{document}

\title{
 Information-Theoretic Regret Bounds \\ for Bandits with Fixed Expert Advice
}

\author{\IEEEauthorblockN{Khaled Eldowa\IEEEauthorrefmark{1}, Nicolò Cesa-Bianchi\IEEEauthorrefmark{1}, Alberto Maria Metelli\IEEEauthorrefmark{2}, Marcello Restelli\IEEEauthorrefmark{2}}
\IEEEauthorblockA{\IEEEauthorrefmark{1}Università degli Studi di Milano, Milan, Italy}
\IEEEauthorblockA{\IEEEauthorrefmark{2}Politecnico di Milano, Milan, Italy}
\IEEEauthorblockA{\{khaled.eldowa, nicolo.cesa-bianchi\}@unimi.it, \{albertomaria.metelli, marcello.restelli\}@polimi.it}
}

\maketitle

\begin{abstract}
We investigate the problem of bandits with expert advice when the experts are fixed and known distributions over the actions. Improving on previous analyses, we show that the regret in this setting is controlled by information-theoretic quantities that measure the similarity between experts. In some natural special cases, this allows us to obtain the first regret bound for EXP4 that
can get arbitrarily close to zero if the experts are similar enough.
While for a different algorithm, we provide another bound that describes the similarity between the experts in terms of the KL-divergence, and we show that this bound can be smaller than the one of EXP4 in some cases.  
Additionally, we provide lower bounds for certain classes of experts showing that the algorithms we analyzed are nearly optimal in some cases.
\end{abstract}


\section{Introduction}
Bandits with expert advice (see, e.g., \cite{lattimore2020bandit}) is a well-known variant of the non-stochastic bandits problem in which, at the beginning of each round, $N$ experts each make a recommendation to the learner in the form of a distribution over the $K$ available actions. The algorithm EXP4 \cite{adversarial} solves this problem with a regret against the best expert bounded by $\sqrt{2TK\log N}$, where $T$ is the horizon. 
When $N \gg K$, this bound shows the ability of EXP4 to leverage the structure of the problem, as opposed to running a bandit algorithm over the $N$ experts achieving a bound of $\sqrt{TN}$. An almost matching lower bound of order $\sqrt{TK\log N/\log K}$ was proved in \cite{expertslowerbound} (for deterministic experts).  
In this work, we study a variant of bandits with expert advice in which the distributions recommended by the experts are \emph{fixed} and \emph{known}.
In the following, we will use the term \emph{policies} to denote these fixed experts.
Our goal is to determine the best possible dependence of the regret on the structure of the policy set $\Theta$ irrespective of the assigned sequence of losses. 

This problem is closely related to linear bandits \cite{linear-bandits} (with finite decision sets), where 
the structure of the decision/policy set can be provably leveraged. 
Our problem can also be viewed as a non-stochastic version of \textit{bandits with mediator feedback} \cite{metelli2021policy}, where the learner's access to actions is mediated by the fixed policy set.
When losses are stochastic rather than being adversarial,
regret bounds were proved in \cite{metelli2021policy} that scale with the largest pair-wise exponentiated 2-Rényi divergence\footnote{This divergence is related to the chi-squared divergence. Note that pairwise, these divergences can be infinite in non-trivial cases, see Example~\ref{example-flower}.} between the policies in the context of policy-based reinforcement learning. Comparable bounds were also proved in \cite{chi-squared} in the setting of contextual bandits.
In our setting, where losses are adversarially generated, the best known bound is $\sqrt{2T\mathcal{S}(\Theta)\log N}$ from \cite{McMahanS09}, where $\mathcal{S}(\Theta) \le \min\{K,N\}$ 
is a notion describing the similarity between policies, see Section \ref{sec:S-interp} for its definition and a new information-theoretic interpretation.
Since $\mathcal{S}(\Theta) \ge 1$ for all $\Theta$, this bound cannot get arbitrarily small no matter how similar the policies are, and becomes vacuous when the policies are identical. 

Our first contribution (Theorem~\ref{upper:alt-losses:thm}) is a new regret bound for EXP4 of the form $\sqrt{2T\mathcal{S}^*(\Theta)\log N}$, where $\mathcal{S}^*(\Theta)$ is a new index of similarity between policies that is never larger than $\mathcal{S}(\Theta)$ and reduces to twice the total variation distance when $N = |\Theta| = 2$. In particular, we show that $\mathcal{S}^*(\Theta)$ can indeed become arbitrarily small, depending on the policy set.
Note that such guarantees cannot be obtained only as a consequence of the reduced range of the losses caused by the similarity of the policies, see \cite{lattimore-lowerbounds}.
Additionally, we show in Theorem~\ref{cvxhull-thm} an algorithm whose regret is bounded by $\sqrt{2 T K D^*(\Theta)}$, where $D^*(\Theta)$ is a notion of the ``width'' of $\Theta$ in terms of the KL-divergence that reduces to the information radius when $\Theta$ is symmetric. This bound is never worse than the EXP4 bound $\sqrt{2TK\log N}$. Moreover, we construct sets $\Theta$ where $K D^*(\Theta) < \mathcal{S}^*(\Theta) \log N$. 
Finally, we prove lower bounds for a number of policy set structures and contrast them with the upper bounds we derived. 
We illustrate, in particular, some examples where the bounds are nearly matching.

\section{Problem Formulation}
We consider a non-stochastic multi-armed bandits problem with a finite action set $\mathcal{A} = [K]$ containing $K$ actions, and a (fixed) policy set $\Theta \in \Delta_{K-1}$ consisting of $N$ probability distributions over the actions. Here $\Delta_{K-1}$ denotes the probability simplex in $\mathbb{R}^K$ and, for a policy $\theta \in \Theta$ and $j \in [K]$, $\theta(j)$ is the probability with which policy $\theta$ picks action $j$. We additionally assume that each arm is in the support of at least one policy.
With a time horizon of $T$ rounds, an instance of the problem is characterized by an unknown sequence of loss vectors $(\ell_t)_{t=1}^T$, where $\ell_t(j) \in [0,1]$, for $j \in [K]$, denotes the loss assigned to action $j$ at round $t \in [T]$. A decision maker interacts with the environment as follows: at each round $t$, the decision maker selects a policy $\theta_t \in \Theta$; an action $A_t \in [K]$ is then sampled from $\theta_t$; the decision maker subsequently suffers the loss $\ell_t(A_t)$ and observes the pair $\big(A_t,\ell_t(A_t)\big)$. With a slight abuse of notation we denote by $\ell_t(\theta)$ the expected loss (at round $t$) given that policy $\theta$ was selected; that is, $\ell_t(\theta) = \sum_{j=1}^K \theta(j) \ell_t(j)$. The objective is to minimize the regret, which we define as follows:
\begin{equation*}
    R_T = \E \sum_{t=1}^T \ell_t(\theta_t) - \sum_{t=1}^T \ell_t(\theta^*),
\end{equation*}
where $\theta^* \in \argmin_{\theta \in \Theta} \sum_{t=1}^T \ell_t(\theta)$ and the expectation is over the internal randomization of the player.  

\section{EXP4 Regret Analysis}
It is possible to show\footnote{See also Theorem 18.3 in \cite{lattimore2020bandit}.} that the EXP4 algorithm with a suitable tuning of the learning rate satisfies a similar regret bound to the one proven for the algorithm developed in \cite{McMahanS09}. In our setting, this bound is $\sqrt{2T\mathcal{S}(\Theta)\log N}$, where
$
    \mathcal{S}(\Theta) = \sum_{j=1}^K \max_{\theta \in \Theta} \theta(j)
$
is a notion of similarity for the policy set. It is easy to see that $\mathcal{S}(\Theta) \leq \min\{K,N\}$, and thus, it can be interpreted as the effective ``number'' of policies.

\begin{algorithm} [t]
    \caption{EXP4 With Fixed Expert Advice}
    \label{exp4}
    \begin{algorithmic}[1]
        \State \textbf{Input:} $K$, $\Theta$, $\eta$
        \State \textbf{Initialize:} $\forall \theta \in \Theta$, $\hat{\ell}_0(\theta)=0$
        \For{$t=1,\dotsc,T$}
            \State Draw $\theta_t \sim P_t$, where $P_{t}(\theta) = \frac{\exp(-\eta\sum_{s=0}^{t-1}\hat{\ell}_s(\theta))}{\sum_{\xi \in \Theta}  \exp(-\eta\sum_{s=0}^{t-1}\hat{\ell}_s(\xi))}$
            \State Draw $A_t \sim \theta_t$, and observe loss $\ell_t (A_t)$
            \State $\forall \theta \in \Theta$, set $\hat{\ell}_t(\theta) = \frac{\theta(A_t)}{\sum_{\xi \in \Theta} P_{t} (\xi) \xi(A_t)} \ell_t(A_t)$
        \EndFor
    \end{algorithmic}
\end{algorithm}

\subsection{An Information-Theoretic Interpretation of \texorpdfstring{$\mathcal{S}(\Theta)$}{S(Theta)}} \label{sec:S-interp}
An alternative characterization of the policy set similarity can be derived by observing that $\mathcal{S}(\Theta) = 1 + \mathcal{TV}(\Theta)$, where for any ordering of the policies $(\theta_i)_{i=1}^N$ we define:
\begin{equation}\label{tv}
    \mathcal{TV}(\Theta) = \sum_{i=2}^N \sum_{j:\theta_i(j) > \theta_{[i-1]}(j)} (\theta_i(j)-\theta_{[i-1]}(j)),
\end{equation}
where $\theta_{[i-1]}(j) = \max_{\theta \in \{\theta_1,\cdots,\theta_{i-1}\}}\theta(j)$, and the second sum is the upper variation of the signed measure $\theta_i-\theta_{[i-1]}$.
$\mathcal{TV}(\Theta)$ can be seen as a generalization of the total variation distance to describe the overall divergence of the policy set. Indeed, it is easy to see that when $\Theta=\{\theta_1, \theta_2\}$,  $\mathcal{TV}(\Theta)$ reduces to the total variation between the two policies:
\[
\tv(\theta_1,\theta_2) = \sum_{j: \theta_2(j) > \theta_1(j)} (\theta_2(j)-\theta_1(j))~.
\]
Moreover, an upper bound on $\mathcal{TV}(\Theta)$ can be derived by noting that for any $\tau \in \Delta_{K-1}$ we have:
\begin{equation*}
    \mathcal{TV}(\Theta) = \mathcal{S}(\Theta) - 1 = \sum_{\theta \in \Theta} \sum_{j \in B(\theta)} (\theta(j) - \tau(j)),
\end{equation*}
where $(B(\theta))_{\theta \in \Theta}$ is any partition of $[K]$ such that for  $j \in B(\theta)$ we have $\theta \in \argmax_{\theta'\in \Theta} \theta'(j)$. Then, it follows that:
\begin{equation}\label{tv-bound}
         \mathcal{TV}(\Theta) \leq \min_{\tau \in \Delta_{K-1}} \sum_{\theta \in \Theta} \tv(\theta,\tau).
\end{equation}
It should be noted that this bound can be loose. For instance, if the policy set is partitioned into clusters of similar policies, then it is not hard to see that the right-hand side of \eqref{tv-bound} will be wasteful compared to \eqref{tv}. It is also noteworthy that quantities related to $\mathcal{S}(\Theta)$ are used when studying the minimax risk in statistical estimation problems. In particular, Theorem II.1 in \cite{minimax-risk} can be used to derive upper bounds similar to \eqref{tv-bound} in terms of any $f$-divergence, though often in implicit form.

\subsection{An Improved Bound}
Nevertheless, the stated regret bound scales with $\mathcal{S}(\Theta)$, not with $\mathcal{TV}(\Theta)$. The main limitation is that 
$\mathcal{S}(\Theta) \geq 1$, no matter how close $\mathcal{TV}(\Theta)$ is to zero. Thus, the bound $\sqrt{2T\mathcal{S}(\Theta)\log N}$ is never smaller than $\sqrt{2T\log N}$ regardless of the structure. One might wonder if this is necessary.
The following theorem provides the first regret bound for EXP4 that can get arbitrarily close to zero if the policies are similar enough. Our bound depends on the key quantity:
\[
        \mathcal{S}^*(\Theta) = \sum_{j=1}^K \left(\max_{\theta \in \Theta} \theta(j) - \min_{\theta' \in \Theta} \theta'(j)\right),
\]
that is easily seen to satisfy
$\mathcal{TV}(\Theta) \le \mathcal{S}^*(\Theta) \le \mathcal{S}(\Theta)$.
\begin{theorem} \label{upper:alt-losses:thm}
    Algorithm \ref{exp4} run with $\eta = \sqrt{ \frac{2\log N}{T \mathcal{S}^*(\Theta)}}$ satisfies
$
        R_T \leq \sqrt{2 T \mathcal{S}^*(\Theta) \log N}.
$
\end{theorem}
\begin{proof}
    For a policy $\theta$, we define a modified version of the loss at time $t$ as $\zeta_t (\theta) = \sum_{j=1}^K (\theta(j) - q(j)) \ell_t(j)$, where $q(j) = \min_{\theta' \in \Theta} \theta'(j)$. Notice that for any two policies $\theta_i$ and $\theta_j$, $\zeta_t(\theta_i) - \zeta_t(\theta_j) = \ell_t(\theta_i) - \ell_t(\theta_j)$. An estimate for this modified loss can be constructed in a standard manner:  $\hat{\zeta}_t (\theta) = (\theta(A_t)-q(A_t)) \frac{\ell_t (A_t)}{\psi_t (A_t)}$, where $\psi_t(j)=\sum_{\theta' \in \Theta} P_{t} (\theta') \theta'(A_t)$ is the probability of playing arm $j$ at time $t$ given $P_t$. Let $\E_{t-1}$ be the expectation conditioned on the events up to round $t-1$, it can be verified\footnote{See the proof of Theorem 4.2 in \cite{bubeck2012regret} for similar arguments.} that (for any $\theta$) $\mathbb{E}_{t-1}\hat{\zeta}_t (\theta)
= \zeta_t (\theta)$ and  $\mathbb{E}_{t-1}\zeta_t (\theta_t)
= \mathbb{E}_{t-1}\sum_{\theta \in \Theta}  P_t(\theta) \hat{\zeta}_t (\theta)$. Hence, we have that
\begin{align}
    R_T &= \E\left[ \sum\nolimits_t \zeta_t (\theta_t) -  \sum\nolimits_t \zeta_t (\theta^*) \right] \nonumber\\ \label{upper:alt-losses:regret}
        &= \E\left[ \sum\nolimits_t \sum\nolimits_{\theta \in \Theta}  P_t(\theta) \hat{\zeta}_t (\theta) -  \sum\nolimits_t \hat{\zeta}_t (\theta^*)\right].
\end{align}
Notice that:
    $
        \sum_{s=1}^{t-1} \hat{\ell}_s (\theta) = \sum_{s=1}^{t-1} \hat{\zeta}_s (\theta) + \sum_{s=1}^{t-1} q(A_s) \frac{\ell_s (A_s)}{\psi_s (A_s)}.
    $
    And since the second sum does not depend on $\theta$, we can rewrite\footnote{Like $\hat{\ell}_0(\theta)$, we initialize $\hat{\zeta}_0(\theta)$ as zero.} $P_t$ as follows:
    \begin{equation*}
        P_{t}(\theta) = \frac{\exp(-\eta\sum_{s=0}^{t-1} \hat{\zeta}_s (\theta))}{\sum_{\theta' \in \Theta}  \exp(-\eta\sum_{s=0}^{t-1} \hat{\zeta}_s (\theta'))}.
    \end{equation*}
    One can then bound (\ref{upper:alt-losses:regret}) using a standard manipulation (see, for example, Theorem 1.5 in \cite{hazan}) to yield that:\footnote{This requires that $\hat{\zeta}_t(\theta)$ is non-negative, which is indeed the case. 
    }
    \begin{equation} \label{upper:alt-losses:hedge}
        R_T \leq  \frac{\eta}{2} \E \sum\nolimits_t \sum\nolimits_{\theta \in \Theta}  P_t(\theta) \hat{\zeta}^2_t (\theta) + \frac{\log N}{\eta}.
    \end{equation}
    Define $s(j) = \max_{\theta \in \Theta} \theta(j)$, we now bound the second moment term at step $t$:
     \begin{align*}
         &\mathbb{E}_{t-1}\sum\nolimits_{\theta\in\Theta}  P_t(\theta) \hat{\zeta}^2_t (\theta)\\ &= \mathbb{E}_{t-1}\sum\nolimits_{\theta\in\Theta} P_t(\theta) (\theta(A_t)-q(A_t))^2 \frac{\ell^2_t (A_t)}{\psi^2_t (A_t)}  \\
         &\leq \mathbb{E}_{t-1}(s(A_t)-q(A_t)) \frac{\ell^2_t (A_t)}{\psi_t (A_t)} \frac{\sum_{\theta\in\Theta} P_t(\theta) (\theta(A_t)-q(A_t)) }{\sum_{\theta\in\Theta} P_t(\theta) \theta(A_t)}  \\
         &\leq \mathbb{E}_{t-1}(s(A_t)-q(A_t)) \frac{\ell^2_t (A_t)}{\psi_t (A_t)}  \\
         & = \sum\nolimits_j (s(j)-q(j)) \ell^2_t (j) \leq \mathcal{S}^*(\Theta).
     \end{align*}
     Combining this with \eqref{upper:alt-losses:hedge} and the specified choice of $\eta$ concludes the proof.
\end{proof}
In general, $\mathcal{S}^*(\Theta)$ is not guaranteed to be strictly smaller $\mathcal{S}(\Theta)$; they can be equal in some cases regardless of how small $\mathcal{TV}(\Theta)$ is. In other cases, however, there can be an improvement. To see this, note that for any $\tau \in \Delta_{K-1}$,
\begin{equation*}
    \mathcal{S}^*(\Theta) = \sum_{\theta \in \Theta} \sum_{j \in B(\theta)} (\theta(j) - \tau(j)) + \sum_{j \in B'(\theta)} (\tau(j)-\theta(j)),
\end{equation*} 
where $(B'(\theta))_{\theta \in \Theta}$ is a partition of $[K]$ such that for every $j \in B'(\theta)$ we have $\theta \in \argmin_{\theta'\in \Theta} \theta'(j)$. Then, analogously to~\eqref{tv-bound}, we have that:
\begin{equation*} 
    \mathcal{S}^*(\Theta) \leq 2\min_{\tau \in \Delta_{K-1}} \sum_{\theta \in \Theta} \tv(\theta,\tau).
\end{equation*}
Like \eqref{tv-bound}, this bound can be loose, but it serves to indicate that if $\min_{\tau \in \Delta_{K-1}} \sum_{\theta \in \Theta} \tv(\theta,\tau)$ is small (it can get arbitrarily so), $\mathcal{S}^*(\Theta)$ is guaranteed to be of at most the same order.

\subsection{Examples}
In the following, we compare the quantities $\mathcal{S}(\Theta)$, $\mathcal{S}^*(\Theta)$, and $\mathcal{TV}(\Theta)$ for a selection of policy set structures. 
\begin{exmp}[Two Policies] \label{example-two}
    As we have seen before, for the two policies case, i.e., $\Theta = \{\theta_1,\theta_2\}$, $\mathcal{S}(\Theta) = 1 + D_{TV}(\theta_1,\theta_2)$. Whereas
$
    \mathcal{S}^*(\Theta) = \sum_{j=1}^K |\theta_1(j) - \theta_2(j)| = 2 D_{TV}(\theta_1,\theta_2).
$
\end{exmp}
The next two examples concern the case in which each policy is a uniform distribution over a support of $M\leq K$ arms. In this scenario, we get that\footnote{Recall that we assume that an arm is in the support of at least one policy.} $\mathcal{S}(\Theta) = \frac{K}{M}$, while $\mathcal{S}^*(\Theta)$ depends on the number of arms common to all policies.

\begin{exmp}[Radially Symmetric Uniform Policies] \label{example-flower}
    Consider a structure where the intersection of the supports of any pair of policies is the same.\footnote{This means that any arm is either in the support of all policies or exclusively in the support of a single one.} 
    Let $V \leq M$ be the number of arms common to all policies, we have that $\mathcal{S}(\Theta) = \frac{N(M-V)+V}{M}$, that is bounded from below by $1$, while $\mathcal{S}^*(\Theta)=N\frac{M-V}{M}$ and $\mathcal{TV}(\Theta)=(N-1)\frac{M-V}{M}$ are not.\footnote{Note that $\frac{M-V}{M}$ is the total variation distance between any two policies.}
    \end{exmp}

\begin{exmp}[Failure of $\mathcal{S}^*$] \label{example-stripes}    
    On the other hand, if $K=M+1$, and the policy set contains all possible $M$-supported uniform policies, then $\mathcal{S}(\Theta) = \frac{M+1}{M}$ which approaches $1$ as $M$ increases, thus $\mathcal{TV}(\Theta)=\frac{1}{M}$ approaches $0$. However, $\mathcal{S}^*(\Theta)$ is always equal to $\mathcal{S}(\Theta)$ since for each arm $j$, $\min_{\theta' \in \Theta} \theta'(j) = 0$. 
\end{exmp}

\begin{exmp}[$\epsilon$-Uniform Policies] \label{example-epsilon}
    If $N=K$ and each policy $\theta$ is associated (one-to-one) with an arm $a_\theta$ so that, for an arm $j$, $\theta(j) = \frac{1-\epsilon}{K} + \epsilon \mathbb{I}\{j=a_\theta\}$, where $0 \leq \epsilon \leq 1$.
    Then, $\mathcal{S}(\Theta) = \epsilon K + 1 - \epsilon$, while $\mathcal{S}^*(\Theta) = \epsilon K$ and $\mathcal{TV}(\Theta) = \epsilon (K-1)$.
\end{exmp}

\section{An Alternative Approach} \label{sec:osmd}
Since we can randomize our policy choice at each round, we can interpret our setting as a bandits problem in which the player has to randomize over the actions choosing a distribution from 
(and also competing with)
the convex hull $\mathrm{co}(\Theta)$ of the available policies. A simple approach, outlined in Algorithm~\ref{osmd}, is to adapt the Online Stochastic Mirror Descent (OSMD) interpretation of EXP3 \cite{adversarial} to our setting. The main distinction is that we need to project onto $\mathrm{co}(\Theta)$ at each step. Denote by $D(P,Q)$ the KL-divergence between distributions $P$ and $Q$, and for $\tau \in \Delta_{K-1}$, define $D(\Theta||\tau)=\max_{\theta\in\Theta} D(\theta, \tau)$. 
The following regret bound for Algorithm~\ref{cvxhull-thm} uses a notion of the ``width'' of $\Theta$ in terms of the KL-divergence defined by:\footnote{Note that the minimum value can only be attained in $\mathrm{co}(\Theta)$; see Theorem 11.6.1 in \cite{cover2006}.}
\[
    D^*(\Theta) =  \min_{\tau\in\Delta_{K-1}} D(\Theta||\tau)~.
\]
\begin{theorem} \label{cvxhull-thm}
Algorithm \ref{osmd} run with
\[
\tau^* \in \argmin_{\tau\in\Delta_{K-1}} D(\Theta||\tau) \quad\text{and}\quad \eta=\sqrt{ \frac{2D^*(\Theta)}{T K}}
\]
satisfies
$
       R_T \leq \sqrt{2TD^*(\Theta)K}
$.
\end{theorem}
\begin{proof}
    Noting  that $\hat{\ell}_t$ is unbiased given $x_t$, we get that
    $
         R_T 
        = \E \sum\nolimits_t \langle x_t - \theta^*, \hat{\ell}_t \rangle.
    $
    This expression is the regret of OMD on the estimated losses with a decision set $\text{co}(\Theta)$ and the negative entropy function $\psi(x) = \sum_{j} x(j) \log x(j)$ as the regularizer. Using Lemma 6.14 in \cite{orabona},\footnote{See also, still in \cite{orabona}, the discussion leading up to Theorem 10.2 concerning the negative entropy regularizer.} we get that:     
    \begin{equation} \label{cvxhull-omd}
        R_T \leq \frac{D(\theta^*, \tau^*)}{\eta} + \frac{\eta}{2} \E \sum\nolimits_t\sum\nolimits_j x_t(j) \hat{\ell}^2_t(j),
    \end{equation}
    Similar to the last step of the proof of Theorem \ref{upper:alt-losses:thm}, we can show that
    $\E \sum_{j} x_t(j) \hat{\ell}^2_t(j) \leq K$. 
    The proof concludes by bounding $D(\theta^*, \tau^*)$ with $D^*(\Theta)$ and plugging in the value of $\eta$.
\end{proof}

It can be shown that $D^*(\Theta)\leq \log N$. Moreover, if the policy set is symmetric, in the sense that the KL-divergence between any policy and the uniform mixture is the same, then $D^*(\Theta)$ (attained at the uniform mixture) coincides with the Jensen-Shannon divergence~\cite{jensen-shannon} (or the information radius) of the policy set.

\begin{algorithm} [thb]
    \caption{OSMD on the Convex Hull of Policies}
    \label{osmd}
    \begin{algorithmic}[1]
        \State \textbf{Input:} $K$, $\Theta$, $\eta$, $\tau \in \text{co}(\Theta): \tau(j)>0 \: \forall j \in [K]$
        \State \textbf{Initialize:} $x_1=\tau$
        \For{$t=1,\dotsc,T$}
            \State Pick distribution $P_t$ on $\Theta$ such that $\sum_{\theta \in \Theta} P_t(\theta) \theta = x_t$
            \State Draw $\theta_t \sim P_t$, then $A_t \sim \theta_t$, and observe loss $\ell_t (A_t)$
            \State $\forall j \in [K]$, set $\hat{\ell}_t(j) = \frac{\mathbb{I}\{j=A_t\}}{x_t(j)} \ell_t(A_t)$
            \State Update $x_{t+1} = \argmin_{x \in \text{co}(\Theta)} \eta \langle x, \hat{\ell}_t \rangle + D(x, x_t)$
        \EndFor
    \end{algorithmic}
\end{algorithm}
Algorithm~\ref{exp4} can also be seen as an instance of OSMD with the negative entropy regularizer. The main difference is that the decision space is the entire probability simplex over the policies. Hence, the regularization in Algorithm~\ref{exp4} favors exploring uniformly over the policies, whereas it favors exploring uniformly over actions in Algorithm~\ref{osmd}. Analysis-wise (compare in particular \eqref{upper:alt-losses:hedge} and \eqref{cvxhull-omd}), Algorithm~\ref{exp4} takes advantage of the similarity between policies to reduce the variance of their loss estimates (compared to $K$), whereas Algorithm~\ref{osmd} takes advantage of the (possibly) limited size of $\text{co}(\Theta)$ to reduce the bias term (compared to $\log N$).

Consider the $\epsilon$-uniform structure of Example \ref{example-epsilon}.
$D^*(\Theta)$ in this case is given by
\begin{equation*} 
    \frac{K-1}{K} (1-\epsilon) \log(1-\epsilon) + \frac{1+\epsilon(K-1)}{K} \log(1+\epsilon(K-1)).
\end{equation*}
If we now compare $K D^*(\Theta)$ and $\epsilon K \log K$, 
we see that both are equal when $\epsilon\in\{0,1\}$. However, the former is strictly convex for $\epsilon\in (0,1)$, while the latter is linear. Thus, in this case, the bound of Theorem \ref{cvxhull-thm} is better than that of Theorem \ref{upper:alt-losses:thm}.
However, if we consider the structure of Example \ref{example-flower}, we have that
$D^*(\Theta) = \frac{M-V}{M} \log N$. Thus, the bound of Theorem \ref{cvxhull-thm} becomes $\sqrt{2T\frac{M-V}{M}K\log N}$ which is worse than the $\sqrt{2T\frac{M-V}{M}N\log N}$ bound of Theorem \ref{upper:alt-losses:thm} since $K \geq N$.

\section{Lower Bounds}
In this section, we prove minimax lower bounds for specific classes of policy sets and contrast them with the regret bounds discussed thus far. More precisely, with a fixed policy set $\Theta$, we prove lower bounds on $\inf_\pi \sup_{(\ell_t)_{t=1}^T} R_T$, 
where $\pi$ is the player's strategy. To this end, we will consider a class of stochastic environments, each identified by the vector $\mu \in [0,1]^K$ such that, for $j \in [K]$ and any $t$, $\ell_t(j)$ is drawn from a Bernoulli distribution with mean $\mu(j)$. For any $t \leq T$, let $H_t = (\theta_1, A_1, \ell_1(A_1), \dotsc, \theta_t, A_t, \ell_t(A_t))$ be the interaction history up to round $T$. Each environment $\mu$ (together with strategy $\pi$) induces a probability distribution $P_\mu$ on $H_T$. Define: 
\begin{equation*}
    R_T(\mu) = \max_{\theta^* \in \Theta} \mathbb{E}_\mu  \sum_{t=1}^T \sum_{j=1}^K (\theta_t(j) - \theta^*(j)) \mu(j),
\end{equation*}
where the subscript in $\mathbb{E}_\mu$ emphasizes the dependence on $P_\mu$. 
In the following, we will prove lower bounds on $\sup_\mu R_T(\mu)$ that hold for any algorithm. 
We will rely on the following lemma, which is an immediate extension of a standard result (see Lemma 15.1 in \cite{lattimore2020bandit}) to our setting.
\begin{lemma} \label{low:kl-decomp}
Fix a strategy $\pi$, a policy set $\Theta$, and a horizon $T$; and let $\mu$ and $\mu'$ be two environments. Then, 
\begin{equation*}
    D(P_\mu, P_{\mu'}) = \sum\nolimits_{\theta \in \Theta} N_\mu(\theta; T) \sum\nolimits_j \theta(j) d(\mu(j),\mu'(j)),
\end{equation*}
where $N_\mu(\theta; T) = \E_\mu\sum_{t=1}^T\mathbb{I}\{\theta_t=\theta\}$, and $d(a,b)$ is the KL-divergence between two Bernoulli distributions with means $a$ and $b$.
\end{lemma}

\subsection{Radially Symmetric Uniform Policies}
The first lower bound concerns the case discussed in Example \ref{example-flower}. The construction of the lower bound (which is an adaptation of the standard approach in \cite{adversarial}) leverages the fact that the policies are uniform and equidistant.
\begin{theorem} \label{low:flower:thm}
If $\Theta$ conforms to the structure of Example \ref{example-flower} such that $M > V$. Then for any algorithm and $T \geq \frac{N M}{4 \log(4/3)(M-V)}$, there exists a sequence of losses such that
$
    R_T \geq \frac{1}{18} \sqrt{N\frac{M-V}{M} T}.
$
\end{theorem}
Since $\mathcal{S}^*(\Theta)=N\frac{M-V}{M}$, it follows that the bound of Theorem \ref{upper:alt-losses:thm} is optimal in this case, up to a logarithmic factor. 
\begin{proof}
    We denote by $U(\theta) \subseteq [K]$ the support of $\theta$ and  $C = \bigcap_{\theta \in \Theta} U(\theta)$. We consider $N$ environments $\{\mu_{\theta}\}_{\theta \in \Theta}$ such that for $\mu_{\theta}$ and arm $j$, $\mu_{\theta}(j) = \frac{1}{2} - \Delta\mathbb{I}\{j \in U(\theta) \backslash C\}$, where $0 < \Delta < \frac{1}{2}$ is to be tuned later. Let $\mu_0$ be an environment such that $\mu_0(j) = \frac{1}{2}$ for any arm $j$. Note that $\theta$ is the optimal policy in $\mu_{\theta}$, and for $\theta' \in \Theta \backslash \{\theta\}$, we have that:
    \begin{equation*}
        \sum_{j=1}^K (\theta'(j) - \theta(j)) \mu(j) = \Delta \sum_{j\in U(\theta) \backslash C} \left(\frac{1}{M} - 0\right) = \Delta \frac{M-V}{M}.
    \end{equation*}
    And thus, 
    \begin{align*}
        R_T(\mu_\theta) &= \Delta \frac{M-V}{M} (T - N_{\mu_\theta}(\theta; T)) \\
        &\geq \Delta \frac{M-V}{M} \bigg(T - N_{\mu_0}(\theta; T) - T \sqrt{\frac{1}{2}D(P_{\mu_0},P_{\mu_\theta})}\bigg),
    \end{align*}
    where the inequality follows from the fact that\footnote{See Exercise 14.4 in \cite{lattimore2020bandit} for a general version of this inequality.} $N_{\mu_\theta}(\theta; T) - N_{\mu_0}(\theta; T) \leq T \tv(P_{\mu_0},P_{\mu_\theta})$ and by Pinsker's inequality. Starting from Lemma \ref{low:kl-decomp}, we have that:
    \begin{align*}
         D(P_{\mu_0}, P_{\mu_\theta}) &= \sum_{\theta' \in \Theta} N_{\mu_0}(\theta'; T) \sum_{j=1}^K \theta'(j) d(\mu_0(j),\mu_\theta(j)) \\
         &= N_{\mu_0}(\theta; T) \sum_{j\in U(\theta) \backslash C} \frac{1}{M} d\left(\frac{1}{2},\frac{1}{2}-\Delta\right)\\
         &= \frac{M-V}{M} N_{\mu_0}(\theta; T) \left(-\frac{1}{2} \log(1-4\Delta^2)\right)\\
         &\leq \frac{M-V}{M} N_{\mu_0}(\theta; T) c \Delta^2, 
    \end{align*}
    where the second equality holds since the only arms whose means have changed between the two environments lie exclusively in the support of $\theta$,
    and the inequality holds for $\Delta \leq \frac{1}{4}$ with $c = 8\log{\frac{4}{3}}$. Hence:
    \begin{align*}
        \sup_{\mu} R_T(\mu) &\geq \frac{1}{N}\sum_{\theta \in \Theta} R_T(\mu_\theta)\\
        &\geq \Delta \frac{M-V}{M} \bigg(T - \frac{T}{N} - T \sqrt{\frac{1}{2}\frac{M-V}{M} \frac{T}{N} c \Delta^2}\bigg)\\
        &\geq \Delta \frac{M-V}{M} T \bigg(\frac{1}{2} - \Delta \sqrt{\frac{1}{2}\frac{M-V}{M} \frac{T}{N} c}\bigg),
    \end{align*}
    where the second inequality holds by the concavity of the square root, and the third since $N\geq2$. The theorem then follows by setting $\Delta = \frac{1}{4} \sqrt{\frac{2MN}{c(M-V)T}}$ and verifying that the condition on $T$ ensures that $\Delta \leq \frac{1}{4}$.
\end{proof}

\subsection{\texorpdfstring{$\epsilon$}{Epsilon}-Uniform Policies}
The next bound concerns the $\epsilon$-Uniform structure of Example \ref{example-epsilon}. While all policies are still equidistant, this case does not enjoy the peculiar discrete structure of Example~\ref{example-flower}.
\begin{theorem} \label{thm:eps-uniform}
    If $\Theta$ conforms to the $\epsilon$-uniform structure of Example \ref{example-epsilon}. Then for $T \geq \frac{K}{4\log(4/3)}$ and any algorithm, there exists a sequence of losses such that
    $
    R_T \geq \frac{1}{18}\epsilon\sqrt{KT}.
    $
\end{theorem}
The proof, see Appendix \ref{appendix-eps-uni-proof}, is similar to that of Theorem~\ref{low:flower:thm} apart from the fact that all policies contribute to the KL-divergence between any two environments (see Lemma~\ref{low:kl-decomp}) 
since all policies have full support.
Notice that the bound is of order $\sqrt{\epsilon^2KT}$ instead of $\sqrt{\epsilon KT} = \sqrt{\mathcal{S}^*(\Theta)T}$, which would have nearly matched the bound of Theorem~\ref{upper:alt-losses:thm}. This was expected since we have shown that Algorithm~\ref{osmd} enjoys a better regret bound in this case. It is interesting to see if a matching lower bound could be proved.  

\subsection{The Two Policies Case}
For the two policies case, we can prove a lower bound of order $\sqrt{H^2(\theta_1,\theta_2)T}$ as asserted by the following theorem, where $H^2(\theta_1,\theta_2)=\frac{1}{2}\sum_j (\sqrt{\theta_1(j)}-\sqrt{\theta_2(j)})^2$ is the squared Hellinger distance. Relative to the total variation, we have that in general:
$
    \frac{1}{2}D^2_{TV}(\theta_1,\theta_2) \leq H^2(\theta_1,\theta_2) \leq D_{TV}(\theta_1,\theta_2).
$ 

\begin{theorem} \label{thm:H2}
    Assume that $\Theta=\{\theta_1,\theta_2\}$ and $H^2(\theta_1,\theta_2) > 0$. Then for any algorithm and $T \geq \frac{1}{8\log(4/3)H^2(\theta_1,\theta_2)}$, there exists a sequence of losses such that
    $
    R_T \geq \frac{1}{13}\sqrt{H^2(\theta_1,\theta_2)T}.$
\end{theorem}

The proof follows a similar scheme as before, so we only sketch the main distinctions and defer the full proof to Appendix \ref{appendix-two-policies-proof}. For an arm $j$, we define $z_1(j) = \frac{\sqrt{\theta_1(j)} - \sqrt{\theta_2(j)}}{\sqrt{\theta_1(j)} + \sqrt{\theta_2(j)}}$, and $z_2(j) = -z_1(j)$.
We use two environments $\mu_1$ and $\mu_2$, where 
$ \mu_1(j) = 1/2 - \Delta z_1(j),
$
with $\mu_2(j)$ defined analogously. 
Subsequently, focusing on $\mu_1$ and $\theta_1$, 
we can bound $R_T(\mu_1)$ from below by
\begin{equation*}
    2 \Delta H^2(\theta_1,\theta_2) \bigg(T - N_{\mu_0}(\theta; T) - T \sqrt{\frac{1}{2}D(P_{\mu_0},P_{\mu_1})}\bigg),
\end{equation*}
whereas we can show that $D(P_{\mu_0}, P_{\mu_1}) \leq 2 c \Delta^2 H^2(\theta_1,\theta_2) T$. To see the latter, it suffices to start from Lemma \ref{low:kl-decomp} and to use that 
$
    d(\mu_0(j),\mu_1(j)) \leq c \Delta^2 z_1(j)^2
$
for sufficiently small $\Delta$, and that 
$
    \sum\nolimits_j \theta_1(j) z_1(j)^2 \leq 2H^2(\theta_1,\theta_2).
$

The squared Hellinger distance can be related to other measures of divergence. For instance, it is of the same order as the Jensen-Shannon divergence and the triangular discrimination \cite{topsoe, H2-JSD}. Thus, the bound of Theorem \ref{thm:H2} can be stated, up to small constants, in terms of these measures as well. 

\subsection{A Matching Lower Bound for a Class of Policy Sets} \label{sec-multi-task}
Lastly, we provide a lower bound that almost matches both Theorems  \ref{upper:alt-losses:thm} and \ref{cvxhull-thm} for a certain class of policy sets that we discuss shortly. This bound is analogous to the $\sqrt{KT\log N/\log K}$ lower bound proved in \cite{expertslowerbound}. However, unlike \cite{expertslowerbound}, we rely on \textit{fixed} sets of \textit{stochastic} policies.
\begin{theorem}
\label{th:worst-case-lower}
For any integer $q\geq2$, there exists a problem structure where $K\geq2$, $\mathcal{S}^*(\Theta) = q$, and $\mathcal{S}^*(\Theta) \log N = K D^*(\Theta)$; such that any algorithm, for sufficiently large $T$, suffers $\Omega\left(\sqrt{\mathcal{S}^*(\Theta) T \frac{\log N}{\log \mathcal{S}^*(\Theta)}}\right)$ regret.
\end{theorem}
In the type of structure referred to in the theorem, the arms are divided into $M$ sections
(where $\frac{K}{M}=q$) and each policy is a uniform distributions supported over $M$ arms such that its support contains an arm from each section. When the policy set contains all such policies 
($(\frac{K}{M})^M$ in total),
this problem becomes equivalent to playing $M$ bandit problems simultaneously with the choice of policy at each round dictating an arm choice at each game. The distinction is that only the loss of one such arm is observed, while the player 
aims to minimize
the average regret of the $M$ games. This type of structure
(albeit with a different type of feedback)
is commonly used to prove lower bounds for combinatorial bandits, see \cite{comb_audibert} for example. 
An adaptation of the proof of Theorem 5 in \cite{comb_audibert} to our case (see Appendix \ref{appendix-multi-task-proof}) leads
to a lower bound of $\Omega(\sqrt{K T})$, from which the theorem follows by using that $\mathcal{S}^*(\Theta) = \mathcal{S}(\Theta)=\frac{K}{M}$ and that $N=(\frac{K}{M})^M$. For what concerns the bound of Theorem~\ref{cvxhull-thm}, we have that
$D^*(\Theta)=\log(\frac{K}{M})$. Thus, Theorems \ref{upper:alt-losses:thm} and \ref{cvxhull-thm} provide the same bound since $\mathcal{S}^*(\Theta) \log N = K D^*(\Theta)$. 

\section{Conclusion}
We analyzed two algorithms providing regret bounds that depend on information-theoretic quantities describing the policy set. We proved lower bounds for certain classes of policy sets highlighting instances where our regret bounds are nearly matched. Nevertheless, it remains to be seen if better guarantees can be proved in cases like Example \ref{example-stripes} where $\mathcal{S}^*(\Theta)\geq1$
even if $\mathcal{TV}(\Theta)$ can be smaller. It is also interesting to see what are the optimal rates for cases like Example~\ref{example-epsilon} where the bound of Theorem~\ref{upper:alt-losses:thm} is suboptimal, as we learned in Section~\ref{sec:osmd}, even if $\mathcal{S}^*(\Theta)$ is of the same order as $\mathcal{TV}(\Theta)$. Another direction is 
to investigate the possibility of adapting
to the policy set structure if the distributions are not known beforehand.


\bibliographystyle{IEEEtran}
\bibliography{IEEE}

\begin{appendices}
\section{Proof of Theorem \ref{thm:eps-uniform}} \label{appendix-eps-uni-proof}
    We will consider $K$ environments $\{\mu_{\theta}\}_{\theta \in \Theta}$ such that for environment $\mu_{\theta}$ and arm $j$, $\mu_{\theta}(j) = \frac{1}{2} - \Delta\mathbb{I}\{j=a_\theta\}$, where $0 \leq \Delta < \frac{1}{2}$ is to be tuned later. Additionally let $\mu_0$ be an environment such that $\mu_0(j) = \frac{1}{2}$ for any arm $j$. Notice that $\theta$ is the optimal policy in $\mu_{\theta}$. Moreover, for $\theta' \in \Theta \backslash \{\theta\}$, we have that
    \begin{equation*}
        \sum_{j=1}^K (\theta'(j) - \theta(j)) \mu(j) = \Delta  (\theta(a_\theta) - \theta'(a_\theta)) = \Delta \epsilon.
    \end{equation*}
    And thus, 
    \begin{align*}
        R_T(\mu_\theta) &= \Delta \epsilon (T - N_{\mu_\theta}(\theta; T)) \\
        &\geq \Delta \epsilon \left(T - N_{\mu_0}(\theta; T) - T \sqrt{\frac{1}{2}D(P_{\mu_0},P_{\mu_\theta})}\right),
    \end{align*}
    where the inequality follows by using that $N_{\mu_\theta}(\theta; T) - N_{\mu_0}(\theta; T) \leq T \tv(P_{\mu_0},P_{\mu_\theta})$ followed by an application of Pinsker's inequality. Starting from Lemma \ref{low:kl-decomp}, we have that 
        \begin{align*}
         &D(P_{\mu_0}, P_{\mu_\theta}) \\
         &= \sum_{\theta' \in \Theta} N_{\mu_0}(\theta'; T) \sum_{j=1}^K \theta'(j) d(\mu_0(j),\mu_\theta(j)) \\
         &=\sum_{\theta' \in \Theta} N_{\mu_0}(\theta'; T) \theta'(a_\theta) d\left(\frac{1}{2},\frac{1}{2}-\Delta\right) \\
         &\leq c \Delta^2 \left(T \frac{1-\epsilon}{K} +   N_{\mu_0}(\theta; T) \epsilon\right),
    \end{align*}
    where for the inequality we used that $d(1/2,1/2-\Delta) = -1/2 \log(1-4\Delta^2) \leq c \Delta^2$ for $\Delta \leq \frac{1}{4}$ and $c = 8\log{\frac{4}{3}}$. Hence,
    \begin{align*}
        \sup_{\mu} R_T(\mu) 
        &\geq \frac{1}{K}\sum_{\theta \in \Theta} R_T(\mu_\theta)\\
        &\geq \Delta \epsilon \left(T - \frac{T}{K} - T \sqrt{\frac{1}{2}c \Delta^2 \left(T \frac{1-\epsilon}{K} +   \frac{T}{K} \epsilon\right)}\right)\\
        &\geq \Delta \epsilon T \left(\frac{1}{2} - \Delta \sqrt{\frac{1}{2} \frac{T}{K} c}\right),
    \end{align*}
    where the second inequality holds by the concavity of the square root, and the third since $K\geq2$. The theorem then follows by setting $\Delta = \frac{1}{4} \sqrt{\frac{2K}{cT}}$ and verifying that the stated condition on $T$ ensures that $\Delta \leq \frac{1}{4}$.

\section{Proof of Theorem \ref{thm:H2}} \label{appendix-two-policies-proof}

We will consider two environments $\mu_1$ and $\mu_2$. For environment $\mu_1$ and arm $j$, we choose\footnote{Note that $\sqrt{\theta_1(j)} + \sqrt{\theta_2(j)}$ is always positive by our assumption that each arm is in the support of at least one policy.} 
\[ \mu_1(j) = \frac{1}{2} - \Delta \frac{\sqrt{\theta_1(j)} - \sqrt{\theta_2(j)}}{\sqrt{\theta_1(j)} + \sqrt{\theta_2(j)}},
\]
where $0 \leq \Delta < \frac{1}{2}$ is to be tuned later. Environment $\mu_2$ is defined analogously. Additionally let $\mu_0$ be an environment such that $\mu_0(j) = \frac{1}{2}$ for any arm $j$. Note that $\theta_1$ ($\theta_2$) is the optimal policy in $\mu_1$ ($\mu_2$). Indeed, 
\begin{align*}
    &\sum_{j=1}^K (\theta_2(j) - \theta_1(j)) \mu_1(j) \\
&\quad= \Delta \sum_{j=1}^K  (\theta_1(j) - \theta_2(j)) \frac{\sqrt{\theta_1(j)} - \sqrt{\theta_2(j)}}{\sqrt{\theta_1(j)} + \sqrt{\theta_2(j)}} \\
&\quad= \Delta \sum_{j=1}^K  (\sqrt{\theta_1(j)} - \sqrt{\theta_2(j)})^2 = 2 \Delta H^2(\theta_1,\theta_2).
\end{align*}
Therefore,  
    \begin{align*}
        &R_T(\mu_1)\\
        &\quad= 2 \Delta H^2(\theta_1,\theta_2) (T - N_{\mu_1}(\theta; T)) \\
        &\quad\geq 2 \Delta H^2(\theta_1,\theta_2) \left(T - N_{\mu_0}(\theta; T) - T \sqrt{\frac{1}{2}D(P_{\mu_0},P_{\mu_1})}\right),
    \end{align*}
    where the inequality follows by using that $N_{\mu_1}(\theta; T) - N_{\mu_0}(\theta; T) \leq T \tv(P_{\mu_0},P_{\mu_1})$ followed by an application of Pinsker's inequality. Note that for $\Delta \leq 1/4$ and $c=8\log(4/3)$, we have that 
    \begin{align*}
    &d(\mu_0(j),\mu_1(j)) \\
    &\quad= d\left(\frac{1}{2},\frac{1}{2}-\Delta \frac{\sqrt{\theta_1(j)} - \sqrt{\theta_2(j)}}{\sqrt{\theta_1(j)} + \sqrt{\theta_2(j)}} \right)\\
    &\quad= -\frac{1}{2} \log\left(1-4\Delta^2\left(\frac{\sqrt{\theta_1(j)}- \sqrt{\theta_2(j)}}{\sqrt{\theta_1(j)} + \sqrt{\theta_2(j)}}\right)^2\right)\\
    &\quad\leq c \Delta^2 \left(\frac{\sqrt{\theta_1(j)}- \sqrt{\theta_2(j)}}{\sqrt{\theta_1(j)} + \sqrt{\theta_2(j)}}\right)^2.
    \end{align*}
    While on the other hand
    \begin{align*}
        &\sum_{j=1}^K \theta_1(j) \left(\frac{\sqrt{\theta_1(j)}- \sqrt{\theta_2(j)}}{\sqrt{\theta_1(j)} + \sqrt{\theta_2(j)}}\right)^2 \\
        &\quad= \sum_{j=1}^K  (\sqrt{\theta_1(j)}- \sqrt{\theta_2(j)})^2 \frac{\theta_1(j)}{(\sqrt{\theta_1(j)} + \sqrt{\theta_2(j)})^2}\\
        &\quad\leq \sum_{j=1}^K  (\sqrt{\theta_1(j)}- \sqrt{\theta_2(j)})^2\\
        &\quad= 2H^2(\theta_1,\theta_2).
    \end{align*}
    With the analogous inequalities for $\mu_2$ and $\theta_2$ we get that 
    \begin{align*}
         D(P_{\mu_0}, P_{\mu_1}) 
         &= 2 c \Delta^2 H^2(\theta_1,\theta_2) (N_{\mu_1}(\theta'; T)+N_{\mu_2}(\theta'; T)) \\
         &= 2 c \Delta^2 H^2(\theta_1,\theta_2)T.
    \end{align*}
    Putting everything together, we get that
        \begin{align*}
        &\sup_{\mu} R_T(\mu)\\ 
        &\quad\geq \frac{1}{2}( R_T(\mu_1)+R_T(\mu_2))\\
        &\quad\geq2 \Delta H^2(\theta_1,\theta_2) \left(T - \frac{T}{2} - T \sqrt{c \Delta^2 H^2(\theta_1,\theta_2)T}\right)\\
        &\quad= \Delta H^2(\theta_1,\theta_2) T \left(1 - 2\Delta \sqrt{c H^2(\theta_1,\theta_2)T}\right),
    \end{align*}
    The theorem then follows by setting $\Delta = \frac{1}{4\sqrt{c H^2(\theta_1,\theta_2)T}}$ and verifying that the stated condition on $T$ ensures that $\Delta$ is less than $1/4$.

\section{Multi-Task Structure Lower Bound}  \label{appendix-multi-task-proof} 
In this section, we prove a lower bound of $\Omega(\sqrt{KT})$ for the multi-task structure described in Section \ref{sec-multi-task}. To reiterate, we have that $K = q M$, where $M \geq 1$ is the number of sections each representing a bandit game, and $q \geq 2$ is the number of arms in each section. We will index the arms according to the section they belong to and their order therein. In other words, $\mathcal{A} = \{a_{i,j}: i\in[M], j\in[q]\}$. The structure of the policy space can be described as follows:
\begin{equation*}
    \Theta = \left\{ \theta \in \Delta^M_K: \forall i \in [M], \sum_{j=1}^q \theta(a_{i,j}) = \frac{1}{M} \right\},
\end{equation*}
where $\Delta^M_K$ is the set of uniform distributions (over $K$ arms) that are supported on only $M$ arms. We will overload the notation and denote by $a_{i,\theta}$ (which belongs to $\{a_{i,j}\}_{j=1}^q$) the arm that is played by policy $\theta$ in section $i$ (i.e. we have that $\theta(a_{i,\theta}) = \frac{1}{M}$).

\begin{theorem} \label{low:multi:thm}
Suppose that the policy and arm spaces conform to the multi-task structure. 
Then for any algorithm and $T \geq \frac{K}{4\log(4/3)}$, there exists a sequence of losses such that
\begin{equation*}
    R_T \geq \frac{1}{18} \sqrt{K T}.
\end{equation*}
\end{theorem}
\begin{proof} 
    We construct an environment $\mu_\theta$ for each policy $\theta$ such that for $a \in \mathcal{A}$, $\mu_{\theta}(a) = \frac{1}{2} - \Delta\mathbb{I}\{a \in U(\theta)\}$, where $U(\theta)$ is the support of $\theta$ and $0 < \Delta < \frac{1}{2}$ is to be tuned later. Moreover, we will also use the following variations of each environment. For $i \in [M]$, let $\mu_\theta^{-i}$ be an environment such that for $a \in \mathcal{A}$, 
    \[ \mu_\theta^{-i}(a) = 
    \begin{cases}
    \frac{1}{2} , & \text{if } a \in \{a_{i,j}\}_{j=1}^q\\
    \mu_{\theta}(a), & \text{otherwise.}
    \end{cases}\]

    For a policy $\theta$, we have that
    \begin{align*}
         R_T(\mu_\theta) &= \mathbb{E}_{\mu_\theta}  \sum_{t=1}^T \sum_{a\in \mathcal{A}} (\theta_t(a) - \theta(a)) \mu_\theta(a) \\ 
         &= \Delta \E\nolimits_{\mu_\theta}\sum_{t=1}^T \sum_{i=1}^M (\theta(a_{i,\theta}) - \theta_t(a_{i,\theta}))  \\
         &= \frac{\Delta}{M} \E\nolimits_{\mu_\theta}\sum_{t=1}^T \sum_{i=1}^M (1 - \mathbb{I}\{a_{i,\theta_t} = a_{i,\theta}\}) \\
         &= \frac{\Delta}{M} \sum_{i=1}^M  (T - N_{\mu_\theta}(i,\theta;T)\}),
    \end{align*}
    where for an environment $\mu$, a policy $\theta$, and a section $i \in [M]$, $N_{\mu}(i,\theta; T) \coloneqq \E_{\mu}\left[\sum_{t=1}^T  \mathbb{I}\{a_{i,\theta_t} = a_{i,\theta}\} \right]$. In words, this counts the expected number of times (under $\mu$) that the chosen policy agrees with $\theta$ in section $i$. Next, we use that, for any $i \in [M]$, $N_{\mu_\theta}(i,\theta;T) - N_{\mu^{-i}_\theta}(i,\theta;T) \leq T \tv(P_{\mu^{-i}_\theta},P_{\mu_\theta})$ together with Pinsker's inequality
    to get that
    \begin{equation} \label{low:multi:reg-kl}
        R_T(\mu_\theta) \geq \frac{\Delta}{M} \sum_{i=1}^M  \bigg(T - N_{\mu^{-i}_\theta}(i,\theta;T) - T\sqrt{\frac{1}{2}D(P_{\mu^{-i}_\theta},P_{\mu_\theta})}\bigg).
    \end{equation}
    As for the KL-divergence term, we apply Lemma \ref{low:kl-decomp}:
    \begin{align*}
         &D(P_{\mu^{-i}_\theta},P_{\mu_\theta})\\
         &\quad= \sum_{\theta' \in \Theta} N_{\mu^{-i}_\theta}(\theta';T) \sum_{a\in\mathcal{A}} \theta'(a) d(\mu^{-i}_\theta(a),\mu_\theta(a)) \\
         &\quad= \sum_{\theta' \in \Theta} N_{\mu^{-i}_\theta}(\theta';T)  \theta'(a_{i,\theta}) d(\mu^{-i}_\theta(a_{i,\theta}),\mu_\theta(a_{i,\theta})) \\
         &\quad= \frac{1}{M}\sum_{\theta' \in \Theta} N_{\mu^{-i}_\theta}(\theta';T)  \mathbb{I}\{a_{i,\theta'} = a_{i,\theta}\} d\left(\frac{1}{2},\frac{1}{2}-\Delta\right) \\
         &\quad\leq \frac{c \Delta^2}{M}\sum_{\theta' \in \Theta} N_{\mu^{-i}_\theta}(\theta';T)  \mathbb{I}\{a_{i,\theta'} = a_{i,\theta}\} \\
         &\quad= \frac{c \Delta^2}{M} \sum_{\theta' \in \Theta: a_{i,\theta'} = a_{i,\theta}} N_{\mu^{-i}_\theta}(\theta';T) \\
         &\quad= \frac{c \Delta^2}{M}   \E\nolimits_{\mu^{-i}_\theta}\sum_{t=1}^T \sum_{\theta' \in \Theta: a_{i,\theta'} = a_{i,\theta}} \mathbb{I}\{\theta_t = \theta'\} \\
         &\quad= \frac{c \Delta^2}{M}   \E\nolimits_{\mu^{-i}_\theta}\sum_{t=1}^T  \mathbb{I}\{a_{i,\theta_t} = a_{i,\theta}\}\\
         &\quad= \frac{c \Delta^2}{M} N_{\mu^{-i}_\theta}(i,\theta;T),  
    \end{align*}
    where the second equality holds since $a_{i,\theta}$ is the only arm that does not have the same mean loss in the two environments, and the inequality holds for $\Delta \leq \frac{1}{4}$ and $c = 8\log{\frac{4}{3}}$. Plugging back into \eqref{low:multi:reg-kl}, we get that
    \begin{align} 
        \nonumber &R_T(\mu_\theta) \\ \label{low:multi:reg-single}
        &\quad\geq \frac{\Delta}{M} \sum_{i=1}^M  \bigg(T - N_{\mu^{-i}_\theta}(i,\theta;T) - T\Delta\sqrt{\frac{c}{2 M} N_{\mu^{-i}_\theta}(i,\theta;T)}\bigg).
    \end{align}
    For what follows, we introduce an extra bit of notation. For $i \in [M]$ and $\theta \in \Theta$, define $F(i,\theta;T) = N_{\mu^{-i}_\theta}(i,\theta;T) + T\Delta\sqrt{\frac{c}{2 M} N_{\mu^{-i}_\theta}(i,\theta;T)}$. Moreover, let $\sim_i$ denote an equivalence relation on the policy set such that for $\theta$, $\theta' \in \Theta$,
    \begin{equation*}
        \theta \sim_i \theta' \iff \forall s \in [M]\backslash\{i\}, a_{s, \theta} = a_{s, \theta'}.
    \end{equation*}
    In words, two policies are equivalent according to $\sim_i$ if they agree everywhere outside section $i$. Denote the set of all equivalence classes of $\sim_i$ by $\Theta / \sim_i$, which contains $q^{M-1}$ classes, each containing $q$ policies corresponding to the possible arm choices in section $i$. Notice that if $\theta \sim_i \theta'$ then $\mu_{\theta}^{-i}$ is the same as $\mu_{\theta'}^{-i}$, and we will thus refer to either of the two environments using the equivalence class to which the two policies belong: $\mu_{[\theta]}^{-i}$.
    Now, with $i$ still referring to a fixed section, we have that
    \begin{align*}
        \sum_{\theta \in \Theta} N_{\mu^{-i}_\theta}(i,\theta;T) &= \sum_{W \in \Theta/\sim_i}
        \sum_{\theta \in W} N_{\mu^{-i}_\theta}(i,\theta;T) \\
        &= \sum_{W \in \Theta/\sim_i}
        \sum_{\theta \in W} N_{\mu^{-i}_W}(i,\theta;T) \\
        &= \sum_{W \in \Theta/\sim_i}
         \E\nolimits_{\mu^{-i}_W}\sum_{t=1}^T \underbrace{\sum_{\theta \in W} \mathbb{I}\{a_{i,\theta^t} = a_{i,\theta}\}}_{=1} \\
         &= q^{M-1} T = q^{M} \frac{T}{q}.
    \end{align*}
    On the other hand,
    \begin{align*}
        \sum_{\theta \in \Theta} \sqrt{  N_{\mu^{-i}_\theta}(i,\theta;T)} &\leq \sqrt{\sum_{\theta \in \Theta} 1^2} \sqrt{\sum_{\theta \in \Theta}  N_{\mu^{-i}_\theta}(i,\theta;T)}  \\
        &= \sqrt{q^M} \sqrt{q^{M} \frac{T}{q}} = q^{M} \sqrt{\frac{T}{q}}.
    \end{align*}
    Thus, we have that
    \begin{align*}
        \sum_{\theta \in \Theta} F(i,\theta;T) \leq q^{M} T\bigg( \frac{1}{q} + \Delta\sqrt{\frac{c T}{2 q M} } \bigg).
    \end{align*}
    Hence, 
    \begin{align*} 
        \sup_\mu R_T(\mu) 
        &\geq \frac{1}{|\Theta|} \sum_{\theta \in \Theta}  R_T(\mu_\theta) \\
        &\geq \frac{1}{|\Theta|} \sum_{\theta \in \Theta}  \frac{\Delta}{M} \sum_{i=1}^M  (T - F(i,\theta;T)) \\
        &\geq \frac{\Delta}{M} \sum_{i=1}^M   \bigg (T - \frac{1}{|\Theta|}  q^{M} T\bigg( \frac{1}{q} + \Delta\sqrt{\frac{c T}{2 q M} } \bigg)\bigg) \\
        &= \Delta T   \bigg (1 -   \frac{1}{q} - \Delta\sqrt{\frac{c T}{2 K} } \bigg) \\
        &\stackrel{q\geq2}{\geq} \Delta T\left(\frac{1}{2} - \Delta\sqrt{\frac{cT}{2K}} \right). 
    \end{align*}
    Plugging $\Delta = \frac{1}{4} \sqrt{\frac{2K}{cT}}$ into the previous display proves the theorem after observing that $\frac{1}{16}\sqrt{\frac{2}{c}} \geq \frac{1}{18}$. Lastly, notice that the condition imposed on $T$ ensures that indeed $\Delta \leq \frac{1}{4}$.
\end{proof}

\end{appendices}

\end{document}